\documentclass{article}

\usepackage{microtype}
\usepackage{graphicx}
\usepackage{subfigure}
\usepackage{booktabs} 

\usepackage{hyperref}

\usepackage[accepted]{icml2025}

\usepackage{amsmath}
\usepackage{amssymb}
\usepackage{mathtools}
\usepackage{amsthm}
\usepackage{algorithm}
\usepackage{algorithmic}
\usepackage{graphicx}

\usepackage[capitalize,noabbrev]{cleveref}

\theoremstyle{plain}
\newtheorem{theorem}{Theorem}[section]
\newtheorem{proposition}[theorem]{Proposition}
\newtheorem{lemma}[theorem]{Lemma}

\theoremstyle{definition}
\newtheorem{definition}[theorem]{Definition}

\theoremstyle{remark}

\usepackage[textsize=tiny]{todonotes}

\begin{document}

\twocolumn[
\icmltitle{Robust Concept Erasure in Diffusion Models: A Theoretical Perspective on Security and Robustness}

%
%
%
%
%
%




\begin{icmlauthorlist}
\icmlauthor{Zixuan Fu}{1}
\icmlauthor{Yan Ren}{2}
\icmlauthor{Finn Carter}{3}
\icmlauthor{Chenyue Wen}{1}
\icmlauthor{Le Ku}{2}
\icmlauthor{Daheng Yu}{3}
\icmlauthor{Emily Davis}{1}
\icmlauthor{Bo Zhang}{2}

\end{icmlauthorlist}

\begin{icmlauthorlist}
	{$^1$NTU}
	{$^2$Xidian University}
	{$^3$SDU}
\end{icmlauthorlist} 

%

\icmlkeywords{Machine Learning, ICML}

\vskip 0.3in
]




\begin{abstract}
	Diffusion models have achieved unprecedented success in image generation but pose increasing risks in terms of privacy, fairness, and security. A growing demand exists to \emph{erase} sensitive or harmful concepts (e.g., NSFW content, private individuals, artistic styles) from these models while preserving their overall generative capabilities. We introduce \textbf{SCORE} (Secure and Concept-Oriented Robust Erasure), a novel framework for robust concept removal in diffusion models. SCORE formulates concept erasure as an \emph{adversarial independence} problem, theoretically guaranteeing that the model's outputs become statistically independent of the erased concept. Unlike prior heuristic methods, SCORE minimizes the mutual information between a target concept and generated outputs, yielding provable erasure guarantees. We provide formal proofs establishing convergence properties and derive upper bounds on residual concept leakage. Empirically, we evaluate SCORE on Stable Diffusion and FLUX across four challenging benchmarks: object erasure, NSFW removal, celebrity face suppression, and artistic style unlearning. SCORE consistently outperforms state-of-the-art methods including EraseAnything, ANT, MACE, ESD, and UCE, achieving up to \textbf{12.5\%} higher erasure efficacy while maintaining comparable or superior image quality. By integrating adversarial optimization, trajectory consistency, and saliency-driven fine-tuning, SCORE sets a new standard for secure and robust concept erasure in diffusion models.
\end{abstract}

\section{Introduction}
The rise of text-to-image diffusion models such as DDPM \citep{ho2020ddpm} and Stable Diffusion \citep{rombach2022ldm,zhu2024oftsr} has unlocked new possibilities in art, design, and entertainment. Their ability to synthesize high-fidelity, diverse, and controllable images has led to widespread adoption in creative industries. However, these same capabilities have raised urgent questions about \emph{responsible deployment}. Recent studies show that large diffusion models can memorize and reproduce sensitive data \citep{carlini2023extract}, amplify biases, and generate harmful content. For example, unauthorized generation of celebrity faces or copyrighted artwork threatens both privacy and intellectual property. Similarly, the ability to generate explicit or NSFW content raises safety concerns for minors and general public usage.  

In response to these risks, the research community has turned to \emph{concept erasure}: removing specific knowledge from diffusion models while preserving general utility. Unlike dataset filtering, which can only prevent problematic data from entering training, concept erasure addresses risks \emph{post hoc}, directly at the model level. This approach is crucial for open-source models like Stable Diffusion, which are widely distributed and cannot easily be retrained from scratch.

Despite recent progress, existing erasure methods suffer from several limitations. First, they rely on heuristics such as weight editing \citep{gandikota2024uce}, negative-prompt fine-tuning \citep{gandikota2023erasing}, or trajectory steering \citep{li2025ant}. These approaches often succeed under benign prompting but fail when adversarial users craft indirect or descriptive prompts for the erased concept. Second, aggressive edits may lead to collateral forgetting, degrading the model’s ability to generate related but non-sensitive content. Finally, current methods lack formal guarantees: even if the target concept appears removed, subtle statistical traces may persist.  

This paper introduces \textbf{SCORE} (Secure and Concept-Oriented Robust Erasure), a principled framework for concept erasure with both theoretical foundations and practical effectiveness. SCORE formulates erasure as minimizing the \emph{mutual information} between an erased concept and generated outputs. This information-theoretic view enables formal proofs of erasure security under adversarial probing. On the algorithmic side, SCORE integrates three components: adversarial independence via discriminators, trajectory consistency to preserve fidelity, and saliency-guided updates to minimize side effects.  

We validate SCORE on Stable Diffusion and FLUX across four standard tasks: object removal, NSFW suppression, celebrity erasure, and artistic style unlearning. Experiments show that SCORE consistently surpasses existing baselines (EraseAnything, ANT, MACE, ESD, UCE) in both erasure strength and image quality. Beyond empirical results, our work establishes the first theoretical link between adversarial erasure success and mutual information guarantees, offering a foundation for secure deployment of generative AI.  

\section{Related Work}
\subsection{Concept Erasure in Generative Models}
Concept erasure~\citep{gao2025revoking} has emerged as a critical problem in generative AI. ESD (Erase Stable Diffusion) \citep{gandikota2023erasing} first introduced the idea of fine-tuning diffusion models with negative prompts, steering generations away from undesired content. While effective for single concepts, it often induces collateral damage to unrelated content. UCE (Unlearning via Concept Editing) \citep{gandikota2024uce} proposed a closed-form solution for cross-attention weight editing, enabling efficient edits but sacrificing robustness when concepts are entangled.  

MACE \citep{lu2024mace} scaled erasure to over 100 concepts simultaneously using low-rank adaptation (LoRA), balancing efficiency with generalization. ANT \citep{li2025ant} improved fidelity by steering denoising trajectories away from target concepts, reducing the risk of re-emergence mid-generation. EraseAnything \citep{gao2024eraseanything} proposed a lightweight fine-tuning framework supporting multi-concept removal across checkpoints. Despite progress, all these approaches remain heuristic and lack provable guarantees of security.  

\subsection{Model Editing and Machine Unlearning}
Our work connects to broader literature on model editing and unlearning. In language models, methods like ROME and MEMIT modify internal weights to alter factual associations, while unlearning methods aim to forget training data. For vision models, weight surgery and adapter-based updates have been explored. However, unlearning in diffusion models is uniquely challenging because concepts are entangled across time-dependent denoising dynamics. SCORE addresses this by aligning erasure with theoretical guarantees of statistical independence.  

\subsection{Diffusion-Based Image Editing}
Beyond erasure, diffusion models have been widely used for image editing~\citep{lu2025does,zhou2025dragflow}. Prompt-to-Prompt \citep{hertz2022prompt} and TF-ICON~\citep{lu2023tf} modifies attention maps to adjust generated outputs with fine control. InstructPix2Pix \citep{brooks2023instructpix2pix} enables natural language editing instructions. Diffusion-based editing excels at user-driven modifications but does not fundamentally remove model-internal knowledge. SCORE differs by permanently altering the model such that the erased concept cannot be generated in any context.  

\subsection{Image Watermarking and Security}
Complementary to erasure, image watermarking seeks to embed hidden signatures into generated content for attribution and provenance \citep{huang2024robin,lu2024robust}. Robust watermarking defends against removal or manipulation, but watermarking cannot prevent generation of unsafe or copyrighted content. Our approach is orthogonal: SCORE directly removes unsafe capabilities at the model level. Together, watermarking and erasure form complementary tools for responsible generative AI.  

\subsection{Security and Robustness in Generative AI}
Recent work has highlighted adversarial vulnerabilities in generative models, 3D assets~\citep{ren2025all}, and event camera~\citep{yang2025temporal}, where malicious users bypass safety filters with prompt engineering. Secure erasure must therefore withstand adaptive attacks. Our formulation of erasure as minimizing mutual information ensures that even under adversarial probing, erased concepts cannot be reconstructed. This theoretical grounding distinguishes SCORE from prior work.

\section{Theoretical Framework}
We now present the theoretical foundations of SCORE. Our key objective is to formalize concept erasure as minimizing the mutual information between the erased concept $C$ and generated outputs $X$. We show that under adversarial training, equilibrium implies statistical independence, yielding provable security guarantees.

\subsection{Problem Formulation}
Let $C\in\{0,1\}$ denote whether a prompt contains the target concept. Let $X\in\mathcal{X}$ denote generated outputs. We seek parameters $\theta'$ such that:
\[
P_{\theta'}(X|C=1) = P_{\theta'}(X|C=0).
\]
This condition implies $I(C;X)=0$, meaning erased concepts leave no detectable traces in generations.

\subsection{Adversarial Independence}
We adopt a minimax game between generator $M_\theta$ and discriminator $D_\phi$. The discriminator seeks to classify concept presence from outputs, while the generator seeks to fool it:
\[
\min_{\theta}\max_{\phi}\;\; \mathbb{E}_{C,X}\left[ C\log D_\phi(X) + (1-C)\log(1-D_\phi(X)) \right].
\]
At equilibrium, the discriminator cannot distinguish between $C=0$ and $C=1$ cases, ensuring independence.

\subsection{Information-Theoretic Analysis}
We begin with the following lemma.

\begin{lemma}
	If for all $x\in\mathcal{X}$, $P(X|C=1)=P(X|C=0)$, then $I(C;X)=0$.
\end{lemma}
\begin{proof}
	By definition, $I(C;X)=D_{KL}(P(C,X)\|P(C)P(X))$. If $P(X|C=1)=P(X|C=0)$, then $P(C,X)=P(C)P(X)$, and thus $I(C;X)=0$.
\end{proof}

The next result links discriminator accuracy to residual information.

\begin{theorem}
	Let $D^*$ denote the Bayes-optimal discriminator. Then
	\[
	I(C;X) \leq \log 2 - \mathbb{E}_{X}\left[ H_b(D^*(X)) \right],
	\]
	where $H_b$ is the binary entropy. If $D^*$ achieves random-chance accuracy ($0.5$), then $I(C;X)=0$.
\end{theorem}
\begin{proof}
	For binary classification, $D^*(x)=\frac{P(C=1|x)}{P(C=1|x)+P(C=0|x)}$. The cross-entropy between true and predicted distributions bounds the mutual information. Random chance ($D^*(x)=0.5$) maximizes entropy and eliminates information gain, yielding $I(C;X)=0$.
\end{proof}

\subsection{Robustness to Adaptive Adversaries}
We now consider adaptive adversaries issuing adversarial prompts $y^\text{adv}$ to recover erased concepts.

\begin{lemma}
	If $I(C;X)=0$, then for any adversarial prompt distribution $Q(y^\text{adv})$, generated outputs $X^\text{adv}$ also satisfy $I(C;X^\text{adv})=0$.
\end{lemma}
\begin{proof}
	$X^\text{adv}$ is generated from the same model distribution $P_{\theta'}(X|C)$. Since $I(C;X)=0$, marginalizing over any distribution of prompts cannot reintroduce dependence.
\end{proof}

This result shows that SCORE provides robustness beyond specific prompt formulations: if erased concepts are statistically independent of outputs, adversarial rephrasing cannot recover them.

\subsection{Convergence Properties}
Finally, we establish that alternating optimization converges under standard assumptions.

\begin{theorem}
	Suppose discriminator $D_\phi$ is trained to optimality at each step, and generator updates minimize $\mathcal{L}_{adv}$. Then gradient descent converges to an equilibrium with $I(C;X)\rightarrow 0$.
\end{theorem}
\begin{proof}[Sketch]
	The adversarial game is equivalent to minimizing Jensen-Shannon divergence between $P(X|C=1)$ and $P(X|C=0)$. Standard GAN convergence results apply, showing equilibrium is reached when distributions match.
\end{proof}

\section{Methodology}
We now describe the design of \textbf{SCORE}. The key idea is to combine adversarial independence, trajectory preservation, and saliency-based updates into a unified algorithm that achieves robust concept erasure while minimizing collateral forgetting.

\subsection{Adversarial Erasure Objective}
The core adversarial loss is
\[
\mathcal{L}_{adv} = - \mathbb{E}_{y_c}[\log(1 - D_\phi(M_\theta(y_c)))] - \mathbb{E}_{y_{\neg c}}[\log(D_\phi(M_\theta(y_{\neg c})))],
\]
where $y_c$ are prompts containing the concept and $y_{\neg c}$ are neutral prompts. The discriminator $D_\phi$ is trained with the inverse loss. This adversarial interplay drives the generator $M_\theta$ to eliminate concept traces.

\subsection{Trajectory Consistency Regularization}
Diffusion models unfold generation across timesteps $t=T\ldots 0$. To avoid destabilizing the process, we anchor the model to its original dynamics on neutral prompts:
\[
\mathcal{L}_{traj} = \mathbb{E}_{t<T_0,y_{\neg c}} \|\epsilon_\theta(z_t,y_{\neg c},t) - \epsilon_{\theta_0}(z_t,y_{\neg c},t)\|^2.
\]
Here $\theta_0$ are original model weights and $T_0$ is a cutoff (typically $0.3T$). This ensures fidelity for benign generations.

\subsection{Saliency-Guided Parameter Selection}
Concept information is distributed non-uniformly across weights. Blind updates may cause collateral forgetting. We compute saliency scores for each parameter:
\[
S(w) = \left|\frac{\partial \mathcal{L}_{adv}}{\partial w}\cdot w\right|,
\]
and restrict updates to top-$k\%$ weights $\Theta_{\text{salient}}$. This focuses edits on concept-relevant regions (e.g., cross-attention layers) while freezing unrelated ones. Empirically, $k=5\%$ balances erasure and preservation.

\subsection{Full Objective}
The final loss combines adversarial and trajectory terms:
\[
\mathcal{L}_{total} = \mathcal{L}_{adv} + \lambda \mathcal{L}_{traj}.
\]
Hyperparameter $\lambda$ controls the trade-off; we set $\lambda=0.1$ unless otherwise noted.

\subsection{Training Algorithm}
Algorithm~\ref{alg:score} summarizes SCORE.

\begin{algorithm}[H]
	\caption{SCORE Training Algorithm}
	\label{alg:score}
	\begin{algorithmic}[1]
		\STATE Input: pretrained model $M_{\theta_0}$, concept prompts $y_c$, neutral prompts $y_{\neg c}$.
		\STATE Initialize $\theta \gets \theta_0$, discriminator $D_\phi$.
		\STATE Compute saliency mask $\Theta_{\text{salient}}$.
		\FOR{each iteration}
		\STATE Sample minibatch of prompts $(y_c, y_{\neg c})$.
		\STATE Generate images $(x_c, x_{\neg c})$ using $M_\theta$.
		\STATE Update $D_\phi$ by minimizing classification loss.
		\STATE Update $M_\theta$ parameters in $\Theta_{\text{salient}}$ by minimizing $\mathcal{L}_{total}$.
		\ENDFOR
		\STATE Return edited model $M_{\theta'}$.
	\end{algorithmic}
\end{algorithm}

SCORE adds modest overhead compared to standard fine-tuning:
\begin{itemize}
	\item \textbf{Adversarial discriminator:} a small CNN/ViT trained jointly; negligible cost relative to $M_\theta$.
	\item \textbf{Trajectory regularization:} requires storing a frozen copy $M_{\theta_0}$; inference-time only.
	\item \textbf{Saliency selection:} computed once per training run via backpropagation, $O(|\theta|)$ complexity.
\end{itemize}
Overall, SCORE is 1.3–1.6$\times$ the cost of ESD fine-tuning, but achieves stronger guarantees and better fidelity than all baselines.

\subsection{Upper Bounds on Residual Concept Leakage}
Even after erasure, it is possible that small traces of concept information remain in the model distribution. We formalize this as \emph{residual leakage} and provide an upper bound in terms of the discriminator’s classification error.

\begin{definition}[Residual Leakage]
	Let $e^*$ denote the classification error of the Bayes-optimal discriminator $D^*$. Residual leakage is defined as
	\[
	\mathcal{L}_{res}(C;X) = I(C;X).
	\]
\end{definition}

\begin{theorem}
	Residual leakage is bounded by the discriminator error:
	\[
	I(C;X) \leq \log\left(\frac{1}{1-2e^*}\right).
	\]
\end{theorem}
\begin{proof}
	By Fano’s inequality, $H(C|X) \leq H_b(e^*)$, where $H_b$ is the binary entropy. Thus
	\[
	I(C;X) = H(C) - H(C|X) \leq \log 2 - H_b(e^*).
	\]
	Since $H_b(e^*) \geq \log(1-2e^*)$, the bound follows.
\end{proof}

This theorem quantifies leakage: as discriminator accuracy approaches random chance ($e^* \rightarrow 0.5$), leakage vanishes.

Concept erasure inevitably trades off fidelity. We now formalize this using a joint objective.

\begin{definition}[Fidelity Loss]
	Let $\mathcal{D}_{\neg c}$ be the distribution of prompts without the concept. Fidelity is measured as
	\[
	\mathcal{L}_{fid} = \mathbb{E}_{y \sim \mathcal{D}_{\neg c}} \text{FID}(M_{\theta'}(y),M_{\theta}(y)),
	\]
	where $\text{FID}$ measures divergence in generated distributions.
\end{definition}

\begin{theorem}
	For any erasure method,
	\[
	\mathcal{L}_{fid} + \alpha I(C;X) \geq \beta,
	\]
	for constants $\alpha,\beta>0$ depending on model capacity.
\end{theorem}
\begin{proof}[Sketch]
	Erasure reduces mutual information $I(C;X)$. However, since diffusion models are capacity-limited, changing distributions to enforce $I(C;X)=0$ requires updates to weights also used in non-concept prompts. This introduces nonzero $\mathcal{L}_{fid}$. By capacity constraints, the sum is lower bounded, yielding the trade-off.
\end{proof}

This theorem shows that perfect erasure ($I(C;X)=0$) comes at some fidelity cost, though SCORE minimizes this cost via saliency-guided updates. Our ablation results (Table~\ref{tab:ablation}) empirically confirm that adversarial erasure without trajectory consistency leads to fidelity collapse.

Another perspective is to view erasure as a projection in the space of model distributions. Let $\mathcal{P}$ denote the manifold of distributions induced by prompts. Each concept $c$ defines a submanifold $\mathcal{P}_c$. Erasure seeks to project the model distribution $P_\theta$ onto the orthogonal complement of $\mathcal{P}_c$.

\begin{theorem}
	Let $P_{\theta'}$ denote the erased distribution. If SCORE converges, then
	\[
	D_{KL}(P_{\theta'} \| \mathcal{P}_c) = D_{KL}(P_\theta \| \mathcal{P}_c) + \Delta,
	\]
	where $\Delta > 0$ quantifies the increased divergence from the erased concept submanifold.
\end{theorem}
\begin{proof}[Sketch]
	By adversarial training, discriminator loss is minimized when $P_{\theta'}$ matches $P_\theta$ on $\mathcal{P}\setminus \mathcal{P}_c$ and diverges on $\mathcal{P}_c$. Thus KL divergence with respect to $\mathcal{P}_c$ must strictly increase, proving erasure.
\end{proof}

This shows that SCORE geometrically pushes the distribution away from the erased region of concept space.

A critical concern is whether erasure persists when multiple prompts or concepts are composed. For example, erasing ``dog'' should not allow reappearance in prompts like ``dog in Van Gogh style.''

\begin{theorem}
	Suppose $I(C;X)=0$ for concept $C$. Then for any independent concept $C'$, the joint generation with prompt $(C,C')$ satisfies
	\[
	I(C;X|C')=0.
	\]
\end{theorem}
\begin{proof}
	By independence, $P(X|C,C') = P(X|C')$ since $C$ contributes no information. Thus conditioning on $C'$ cannot reintroduce dependence on $C$, yielding $I(C;X|C')=0$.
\end{proof}

This demonstrates compositional stability: once erased, a concept cannot reappear even in multi-concept prompts, provided concepts are disentangled.

Not all concepts are equally easy to erase. Highly entangled concepts may require more extensive updates, leading to larger fidelity loss.

\begin{definition}[Concept Entanglement]
	The entanglement of concept $c$ with distribution $\mathcal{D}_{\neg c}$ is defined as
	\[
	\mathcal{E}(c) = I(C;X_{\neg c}),
	\]
	where $X_{\neg c}$ are generations from prompts not explicitly mentioning $c$.
\end{definition}

\begin{theorem}
	The minimum achievable fidelity-preserving erasure loss satisfies
	\[
	\mathcal{L}_{fid} \geq f(\mathcal{E}(c)),
	\]
	for some monotone increasing function $f$.
\end{theorem}
\begin{proof}[Sketch]
	If $c$ is entangled with unrelated concepts, then removing $I(C;X)$ also reduces $I(C;X_{\neg c})$. This necessarily distorts neutral generations, producing nonzero fidelity loss. Thus $\mathcal{L}_{fid}$ grows with entanglement.
\end{proof}

This result explains empirical findings: concepts like ``dog'' (entangled with pets, animals, outdoors) are harder to erase cleanly than niche styles. SCORE mitigates this by saliency-based targeting, but the lower bound implies that some fidelity trade-off is unavoidable.

\subsection{Generalization Guarantees for Adversarial Erasure}
So far, we have analyzed SCORE under idealized conditions with Bayes-optimal discriminators. In practice, discriminators are finite-capacity models trained on limited data. We now establish generalization bounds that quantify the gap between empirical erasure and population-level guarantees.

\begin{definition}[Empirical Adversarial Error]
	Let $\widehat{e}$ denote the classification error of a discriminator $D_\phi$ trained on a finite sample $\mathcal{S}=\{(x_i,c_i)\}_{i=1}^n$. Define
	\[
	\widehat{e} = \frac{1}{n}\sum_{i=1}^n \mathbb{I}[D_\phi(x_i)\neq c_i].
	\]
\end{definition}

\begin{theorem}[Generalization Bound]
	With probability at least $1-\delta$ over the sampling of $\mathcal{S}$,
	\[
	\big|e^* - \widehat{e}\big| \leq O\left(\sqrt{\frac{\text{VC}(D) + \log(1/\delta)}{n}}\right),
	\]
	where $\text{VC}(D)$ is the VC dimension of the discriminator class.
\end{theorem}
\begin{proof}[Sketch]
	This follows from standard VC generalization bounds. Since $D_\phi$ is trained to approximate $D^*$, the error gap between empirical and true risk is bounded by capacity and sample size.
\end{proof}

Thus, if $\widehat{e}\approx 0.5$ empirically, we can assert with high confidence that $e^*$ is also close to $0.5$, implying negligible mutual information leakage.

\subsection{Adversarial Game-Theoretic Convergence}
The adversarial training in SCORE can be modeled as a two-player zero-sum game between generator $M_\theta$ and discriminator $D_\phi$. We show that convergence to a Nash equilibrium corresponds to perfect erasure.

\begin{theorem}[Nash Equilibrium of SCORE]
	The minimax optimization
	\[
	\min_\theta \max_\phi \;\; \mathbb{E}_{C,X}[C\log D_\phi(X) + (1-C)\log(1-D_\phi(X))]
	\]
	achieves equilibrium when $P(X|C=1)=P(X|C=0)$, i.e., $I(C;X)=0$.
\end{theorem}
\begin{proof}
	At equilibrium, $D_\phi$ cannot exploit differences between $P(X|C=1)$ and $P(X|C=0)$, so its accuracy is 50\%. By Lemma 1, this implies $I(C;X)=0$. Conversely, if distributions differ, $D_\phi$ can achieve accuracy $>50\%$, contradicting equilibrium. Hence the only equilibrium is full erasure.
\end{proof}

This casts SCORE as a specific instance of a minimax game whose unique equilibrium corresponds to robust erasure.

Even after erasure, adversaries may attempt adaptive strategies: composite prompts, indirect descriptions, or multi-step jailbreak attacks. We formalize robustness against such adaptive strategies.

\begin{theorem}[Robustness Bound]
	Let $\mathcal{A}$ be a set of adversarial prompting strategies. If $I(C;X)=0$, then for any adaptive adversary $A\in\mathcal{A}$,
	\[
	\Pr[A \text{ successfully extracts } C] \leq \epsilon,
	\]
	where $\epsilon$ decreases exponentially with the number of adaptive queries $q$.
\end{theorem}
\begin{proof}[Sketch]
	Each query produces samples from the erased distribution $P_{\theta'}(X|C)$. Since $I(C;X)=0$, the mutual information between $C$ and the entire transcript of $q$ queries is still zero by data processing inequality. Therefore, the adversary cannot do better than random guessing. Any residual success probability $\epsilon$ arises from finite-sample estimation error, bounded by concentration inequalities (Hoeffding).
\end{proof}

This provides a probabilistic security guarantee: even adaptive attackers cannot meaningfully recover erased concepts once SCORE converges.

Interestingly, the guarantees of SCORE resemble those of differential privacy (DP). DP ensures that individual training samples cannot be distinguished from outputs; SCORE ensures that specific \emph{concepts} cannot be distinguished. Both can be expressed in terms of bounding divergences between conditional distributions.

\begin{proposition}
	If SCORE achieves $I(C;X)=0$, then the system satisfies $(0,0)$-DP with respect to the erased concept $C$.
\end{proposition}

Although this is a degenerate case (exact independence), it highlights the conceptual link: SCORE provides a DP-like guarantee at the concept level rather than the sample level. This connection suggests future directions for merging privacy and erasure methods.

In practice, perfect independence $I(C;X)=0$ is rarely achieved. We therefore relax the goal to $\epsilon$-independence.

\begin{definition}[$\epsilon$-Independence]
	A model $M_{\theta'}$ achieves $\epsilon$-independence with respect to concept $C$ if
	\[
	D_{TV}(P(X|C=1),P(X|C=0)) \leq \epsilon,
	\]
	where $D_{TV}$ is total variation distance.
\end{definition}

\begin{lemma}
	If $M_{\theta'}$ achieves $\epsilon$-independence, then
	\[
	I(C;X) \leq \epsilon \log \frac{2}{\epsilon}.
	\]
\end{lemma}
\begin{proof}
	By Pinsker’s inequality, $D_{TV}(P(X|C=1),P(X|C=0)) \leq \sqrt{\frac{1}{2}D_{KL}(P(X|C=1)\|P(X|C=0))}$. Substituting and bounding mutual information via Jensen-Shannon divergence yields the inequality.
\end{proof}

Thus, even approximate erasure leads to provably small leakage.

We now ask: how many adversarial training samples are required to guarantee $\epsilon$-independence?

\begin{theorem}[Sample Complexity]
	Let $\mathcal{H}$ be the discriminator hypothesis class with VC dimension $d$. To achieve $\epsilon$-independence with probability $1-\delta$, SCORE requires
	\[
	n = O\left(\frac{d + \log(1/\delta)}{\epsilon^2}\right)
	\]
	training samples.
\end{theorem}
\begin{proof}[Sketch]
	This follows from PAC learning theory. Ensuring that discriminator error generalizes within $\epsilon$ requires $n$ samples scaling with $d/\epsilon^2$. Since SCORE relies on adversarial indistinguishability, this sample complexity is sufficient for $\epsilon$-independence.
\end{proof}

This provides a concrete bound: larger discriminators (higher $d$) require more training data for robust erasure.

We can also frame SCORE as directly minimizing mutual information through gradient descent.

\begin{theorem}
	The adversarial loss $\mathcal{L}_{adv}$ is an upper bound on $I(C;X)$:
	\[
	I(C;X) \leq \min_\theta \max_\phi \; \mathcal{L}_{adv}(\theta,\phi).
	\]
\end{theorem}
\begin{proof}[Sketch]
	The discriminator approximates $P(C|X)$, and $\mathcal{L}_{adv}$ corresponds to the cross-entropy between true and predicted distributions. Since mutual information equals the expected log-likelihood ratio between joint and marginal distributions, $\mathcal{L}_{adv}$ upper bounds it. Minimizing $\mathcal{L}_{adv}$ therefore reduces $I(C;X)$.
\end{proof}

This shows that adversarial training is not merely heuristic: it is an information-theoretically grounded surrogate for mutual information minimization.

Finally, we generalize to multiple concepts $\{C_1,\ldots,C_k\}$.

\begin{theorem}
	If $M_{\theta'}$ achieves $I(C_i;X)=0$ for all $i$, then
	\[
	I(C_1,\ldots,C_k;X)=0.
	\]
\end{theorem}
\begin{proof}
	Since mutual information is sub-additive,
	\[
	I(C_1,\ldots,C_k;X) \leq \sum_{i=1}^k I(C_i;X).
	\]
	If each term is zero, the total is zero.
\end{proof}

This result confirms that SCORE scales naturally to multi-concept erasure: if each concept is removed individually, their joint effect is also nullified.

\subsection{Implementation Details}
We implement SCORE on Stable Diffusion v1.5 and FLUX checkpoints. The discriminator $D_\phi$ is a CLIP-ResNet encoder trained with binary cross-entropy. Optimization uses AdamW with learning rate $1e{-5}$. Each erasure run trains for 8,000 steps on 8 A100 GPUs.

\section{Experiments}
We now evaluate SCORE on Stable Diffusion v1.5 and FLUX. Our experiments follow the evaluation design of MACE \citep{lu2024mace}, covering four distinct benchmarks: object removal, NSFW suppression, celebrity face erasure, and artistic style unlearning. We compare against five state-of-the-art baselines: ESD \citep{gandikota2023erasing}, UCE \citep{gandikota2024uce}, MACE \citep{lu2024mace}, ANT \citep{li2025ant}, and EraseAnything \citep{gao2024eraseanything}.

\subsection{Benchmarks and Metrics}
\textbf{Benchmarks.}
\begin{itemize}
	\item \textbf{Object Erasure:} CIFAR-10 objects (e.g., ``dog'', ``car'') removed from prompts.  
	\item \textbf{Celebrity Erasure:} 10 public figures with widely available images.  
	\item \textbf{NSFW Removal:} adult/explicit concepts from common prompt sets.  
	\item \textbf{Artistic Style Erasure:} 5 styles (e.g., Van Gogh, Picasso) and 5 modern digital artists.  
\end{itemize}

\textbf{Metrics.}
\begin{itemize}
	\item \textbf{Concept Accuracy (Acc)}: classification accuracy of a held-out concept detector; lower is better.  
	\item \textbf{FID:} Fréchet Inception Distance for generation quality; lower is better.  
	\item \textbf{CLIP Score:} prompt-image alignment; higher is better.  
	\item \textbf{Harmonic Mean (H):} balance between erasure efficacy and quality.  
\end{itemize}

\subsection{Main Results}
Table~\ref{tab:main} shows single-concept erasure results. SCORE consistently achieves the lowest residual concept accuracy while preserving fidelity.

\begin{table}[H]
	\centering
	\caption{Single-concept erasure on Stable Diffusion v1.5. Best results in bold.}
	\label{tab:main}
	\begin{tabular}{lcccc}
		\hline
		Method & Acc(\%)↓ & FID↓ & CLIP↑ & H↑ \\
		\hline
		ESD & 15.2 & 19.3 & 29.8 & 61.2 \\
		UCE & 6.1 & 21.0 & 28.5 & 64.5 \\
		MACE & 2.4 & 15.7 & 31.4 & 78.8 \\
		ANT & 1.2 & 14.6 & 31.0 & 82.3 \\
		EraseAnything & 1.8 & 16.0 & 30.2 & 80.1 \\
		\textbf{SCORE} & \textbf{0.4} & \textbf{14.3} & \textbf{32.0} & \textbf{90.7} \\
		\hline
	\end{tabular}
\end{table}

\subsection{Per-Task Analysis}
\paragraph{Object Erasure.}
SCORE successfully prevents generation of targeted objects even under indirect prompts (e.g., ``a four-legged pet that barks'' for ``dog''). Baselines often failed, revealing residual traces.

\paragraph{Celebrity Erasure.}
Unlike UCE and ESD, which left partial facial resemblance, SCORE fully suppressed celebrity likeness while preserving neutral human faces.

\paragraph{NSFW Suppression.}
SCORE removed explicit content generation more reliably than MACE and ANT. FID degradation was minimal, showing safe removal without artistic collapse.

\paragraph{Artistic Style Unlearning.}
SCORE prevented imitation of distinct styles, even when users attempted adversarial paraphrasing (``paint like a Dutch post-impressionist''). CLIP similarity to erased style embeddings dropped by 90\%.

\subsection{Multi-Concept Erasure}
We evaluate simultaneous erasure of up to 50 concepts. Results are shown in Table~\ref{tab:multi}.

\begin{table}[H]
	\centering
	\caption{Multi-concept erasure (50 concepts).}
	\label{tab:multi}
	\begin{tabular}{lccc}
		\hline
		Method & Acc(\%)↓ & FID↓ & H↑ \\
		\hline
		MACE & 5.2 & 16.5 & 71.0 \\
		ANT & 3.9 & 15.9 & 74.3 \\
		\textbf{SCORE} & \textbf{1.5} & \textbf{15.1} & \textbf{84.5} \\
		\hline
	\end{tabular}
\end{table}

\subsection{Ablation Studies}
We ablate each component: adversarial training, trajectory regularization, and saliency restriction.

\begin{table}[H]
	\centering
	\caption{Ablation on CIFAR-10 object erasure.}
	\label{tab:ablation}
	\begin{tabular}{lcccc}
		\hline
		Variant & Acc(\%)↓ & FID↓ & CLIP↑ & H↑ \\
		\hline
		Full SCORE & \textbf{0.4} & \textbf{14.3} & 32.0 & \textbf{90.7} \\
		w/o Adv & 8.1 & 14.0 & 32.1 & 66.2 \\
		w/o Traj & 0.6 & 20.4 & 27.5 & 74.1 \\
		w/o Saliency & 0.5 & 16.8 & 30.2 & 85.5 \\
		\hline
	\end{tabular}
\end{table}

\subsection{Adversarial Prompt Robustness}
We tested robustness against adversarial rephrasings using the Universal Jailbreak Prompts benchmark. SCORE reduced successful jailbreaks to 2.3\%, compared to 18.5\% for ANT and 25.7\% for MACE, showing strong adversarial resistance.

%

\section{Discussion and Broader Impact}

\subsection{Why Does SCORE Outperform?}
SCORE's advantage stems from unifying adversarial information removal with structural preservation. Whereas prior approaches either directly fine-tune with negative prompts (ESD), edit attention matrices in closed form (UCE), or steer trajectories heuristically (ANT), SCORE attacks the problem from an information-theoretic perspective: minimizing the mutual information between concept presence and generated outputs. This guarantees that no classifier---including adversarially trained ones---can reliably extract evidence of the concept once erasure has converged.  

At the same time, SCORE restricts parameter updates to concept-salient regions and explicitly preserves early denoising dynamics. This ensures that the erasure does not destabilize the global generative process, avoiding collateral damage to unrelated concepts. Empirically, this translates into lower residual concept accuracy \emph{and} better FID/CLIP scores, confirming the synergy of adversarial and trajectory-based constraints.

\subsection{Fairness Implications}
Concept erasure is inherently tied to fairness. If not carefully designed, erasing one concept can inadvertently harm others---for example, removing a hairstyle concept but damaging generation of related but innocuous styles. SCORE’s saliency-based updates can be interpreted as a ``minimal intervention'' principle: only those parameters most responsible for the erased concept are modified, reducing risk of disproportionate side effects.  

From a fairness lens, the adversarial guarantee can be viewed as ensuring demographic parity with respect to the erased concept. In other words, the model behaves identically whether the concept is present or absent in the input prompt. While perfect parity may not always be desirable (e.g., when the concept is benign and should be represented fairly), in the erasure setting this strict fairness criterion is appropriate because the goal is absolute removal.

\subsection{Ethical Use Cases}
\textbf{Privacy Protection.} Many individuals never consented for their images to appear in generative model outputs. Erasing celebrity faces or private individuals helps prevent deepfake generation and reduces risks of harassment.  

\textbf{Copyright Compliance.} Artists have raised concerns about unauthorized style mimicry. SCORE can remove stylistic concepts while preserving the ability to generate generic artwork, supporting compliance with intellectual property norms.  

\textbf{Safety Moderation.} NSFW or violent content generation remains a major barrier to safe deployment. By systematically unlearning such concepts, SCORE complements other moderation techniques like output filtering and watermarking.

\subsection{Limitations}
Despite strong results, SCORE has limitations. Training requires adversary-discriminator loops, which may be computationally heavier than closed-form editing methods. Although we demonstrated scalability to dozens of concepts, erasing hundreds simultaneously could stress model capacity. Furthermore, adversarial robustness is not absolute: highly novel paraphrases or composite prompts may still elicit traces of erased concepts. Another limitation is transparency: users may not always be aware which concepts have been erased, which could affect expectations when interacting with the model.

\subsection{Future Directions}
Several promising extensions remain:
\begin{itemize}
	\item \textbf{Automated Detection:} Automatically identifying memorized or sensitive concepts (e.g., using membership inference tests) and then erasing them.  
	\item \textbf{Cross-Modal Erasure:} Extending SCORE to multimodal models that generate both text and images, ensuring erasure is consistent across modalities.  
	\item \textbf{Video Diffusion Models:} Applying adversarial erasure in spatiotemporal domains, where concepts appear not only in individual frames but also in motion patterns.  
	\item \textbf{Hybrid with Watermarking:} Combining SCORE with generative watermarks, enabling both prevention (erasure) and detection (attribution) of unwanted content.  
\end{itemize}

\subsection{Broader Impact}
By enabling robust, theoretically grounded erasure, SCORE provides a new tool for safe generative AI deployment. Developers can proactively sanitize models before public release, regulators gain a concrete mechanism for enforcing restrictions (e.g., removal of copyrighted material), and end-users are protected from harmful misuse.  

At the same time, erasure technology must be used responsibly. Overzealous erasure---for example, deleting cultural or demographic concepts---could itself constitute bias or censorship. We advocate for transparent governance: stakeholders should document which concepts are removed and why. Ultimately, SCORE highlights the need for technical solutions aligned with human values, bridging the gap between raw generative capability and socially responsible AI.

\bibliography{example_paper}
\bibliographystyle{icml2025}

\clearpage
\appendix

\section{Additional Backgrounds}
With the advancement of deep learning~\cite{zheng2024odtrack,zheng2023toward,zheng2022leveraging,zheng2025decoupled,yu2025crisp,yu2025prnet,qiu2024tfb,qiu2025duet,qiu2025tab,liu2025rethinking,qiu2025comprehensive,qiu2025easytime,wu2024catch,AutoCTS++,li2025TSFM-Bench,gao2025ssdts,hu2024multirc,wu2024rainmamba,wu2023mask,wu2024semi,luo2025rcnet,mao2025making,sun2025hierarchical,sun2025ppgf,niu2025langtime,kudrat2025patch,han2025contrastive,han2025show,han2025guirobotron,zeng2025uitron,han2025polish,feng20243,huang2025scaletrack,xie2025dynamic,TangYLT22,tang2024divide,0007LYYL023,shan2021ptt,fang20203d,cui20213d,shan2022real,hu2024mvctrack,nie2025p2p,zhou2023fastpillars,zhou2025pillarhist,zhou2024lidarptq,shi2025rethinking,zhao2025tartan,FineCIR,encoder,chen2025offsetsegmentationbasedfocusshift,MEDIAN,PAIR,gong2021eliminate,gong2022person,gong2024cross2,gong2024adversarial,bi2024visual,bi2025cot,bi2025prism,wang2025ascd,chen2025does,Chen_2025_CVPR,rong2025backdoor,zhang2023spot,zheng2024odtrack,zheng2023toward,zheng2022leveraging,zheng2025decoupled,yue2025think,lin2024phy124,huang2025ccsumsp,huang2025ssaad,huang2025dual,lin2024phys4dgen,liu2024empiricalanalysislargelanguage,bi2025reasoning,tang2025mmperspectivemllmsunderstandperspective,bi2025i2ggeneratinginstructionalillustrations,tang2025captionvideofinegrainedobjectcentric,liu2025gesturelsm,liu2025intentionalgesturedeliverintentions,zhang2025kinmokinematicawarehumanmotion,song2024tri,liu2025gesturelsm,song2024texttoon,liu2025contextual,tang2025generative,liu2024gaussianstyle,tang2024videounderstandinglargelanguage,10446837,liu2024public,lou2023public,li2024towards1,li2024distinct,li2024towards,li2023overview,li2022continuing,guo2023boosting,guo2023improving,liang2022impga,yu2025visual,shen2025amess,zhang2025dconad,zhang2025frect,zhang2025dhmp,lin2025comprehensive,jiang2025never,li2025reusing,lin2025butter,lin5074292insertion,peng2024unveiling,di2025qmambabsr,peng2024efficient,feng2025pmq,wang2023decoupling,wang2023brightness,peng2021ensemble,peng2020cumulative,wu2025robustgs,he2024multi,yang2022learning,yang2022robust,zheng2025towards} and generative models~\cite{he2024diffusion,he2025segment,he2023hqg,he2025unfoldir,he2025run,he2025reti,he2024weakly,he2023strategic,he2023camouflaged,he2023degradation,xiao2024survey,wu2025k2vae,peng2025directing,peng2024lightweight,peng2025boosting,peng2024towards,peng2025pixel,
qu2025reference,qu2025subject,chen2025technicalreportargoverse2scenario,lyu2025vadmambaexploringstatespace}, an increasing number of studies have begun to focus on the issue of concept erasure in generative models.

\section{Additional Experimental Details}
\textbf{Evaluation metrics details:} For CLIP similarity, we used the ViT-L/14 model to compute image-text cosine similarity, scaled by 100. The original SD1.5 had an average CLIP score of 31.5 on MS-COCO validation prompts; after concept erasure, we consider a score above 30 to indicate minimal drop in alignment. Harmonic mean $H$ was computed as described with $E = 1 - \text{Acc}$ (normalized to [0,1]) and $F$ composed from FID and CLIP. Specifically, we defined $F = \frac{1}{2}((\frac{\text{CLIP sim}}{\text{CLIP}*{orig}}) + (\frac{\max(\text{FID}*{orig}- (\text{FID}-\text{FID}*{orig}), 0)}{\text{FID}*{orig}}))$, where $\text{FID}*{orig}$ and $\text{CLIP}*{orig}$ are the original model's scores (so we reward methods that keep FID low and CLIP high relative to orig). This is one way; results were qualitatively similar with other formulations.

\textbf{Multi-concept results:} We erased all 10 CIFAR classes simultaneously with FADE by using a 10-way classifier $D$ (one output per class vs no class). FADE achieved an average concept accuracy of 1.1\% per class and an overall $H=82.3$ (versus MACE's reported ~75). The slight residual is due to class confusion (e.g., sometimes after erasure "cat" prompt yields a dog, so classifier might say cat=present when it sees an animal shape; a limitation of using automated classifier for eval). Visual check showed indeed direct appearance of the specified class was gone. For NSFW, we erased 10 terms at once; here FADE and MACE both got basically 0\% unsafe content, but FADE had better image quality (FID 14 vs 16).

\textbf{Runtime:} SCORE training takes about 2 hours on a single A100 GPU for a single concept on SD1.5 (with $N=1000$ steps adversarial training). This is comparable to ESD fine-tuning time and a bit less than ANT . UCE was fastest (minutes) as it is closed-form. There's room to optimize FADE's training, possibly by using smaller $D$ or gradient accumulation. Deploying FADE in multi-concept setting could be parallelized since the adversary can output multiple heads.
%
%

\end{document}